\documentclass[
]{ceurart}

\bibstyle{alpha}
\usepackage{natbib}

\usepackage{amsmath}
\usepackage{amssymb}
\usepackage{amsthm}
\usepackage{array}
\usepackage{booktabs}
\usepackage{graphicx}
\usepackage{subcaption}

\newtheorem{theorem}{Theorem}%
\newtheorem{definition}{Definition}%

\usepackage{listings}
\begin{document}

\copyrightyear{2024}
\copyrightclause{Copyright for this paper by its authors. Use permitted under Creative Commons License Attribution 4.0 International (CC BY 4.0).}

\conference{AEQUITAS 2024: Workshop on Fairness and Bias in AI | co-located with ECAI 2024, Santiago de Compostela, Spain}

\title{Data Quality Dimensions for Fair AI}

\tnotemark[1]

\author[1]{Camilla Quaresmini}[%
orcid=0000-0002-6474-1284,
email=camilla.quaresmini@polimi.it,
]
\cormark[1]
\fnmark[1]
\address[1]{Department of Electronics, Information and Bioengineering, Politecnico di Milano, Piazza Leonardo da Vinci 32, 20133, Milan, Italy}

\author[2]{Giuseppe Primiero}[%
orcid=0000-0003-3264-7100,
email=giuseppe.primiero@unimi.it,
]
\fnmark[1]
\address[2]{LUCI Lab and PhilTech Research Center, Department of Philosophy, Università degli Studi di Milano, Via Festa del Perdono 7, 20122, Milan, Italy and MIRAI, Srl.}

\cortext[1]{Corresponding author.}
\fntext[1]{These authors contributed equally.}

\begin{abstract}
Artificial Intelligence (AI) systems are not intrinsically neutral and biases trickle in any type of technological tool. In particular when dealing with people, the impact of AI algorithms' technical errors originating with mislabeled data is undeniable. As they feed wrong and discriminatory classifications, these systems are not systematically guarded against bias. In this article we consider the problem of bias in AI systems from the point of view of data quality dimensions. We highlight the limited model construction of bias mitigation tools based on accuracy strategy, illustrating potential improvements of a specific tool in gender classification errors occurring in two typically difficult contexts: the classification of non-binary individuals, for which the label set becomes incomplete with respect to the dataset; and the classification of transgender individuals, for which the dataset becomes inconsistent with respect to the label set. Using formal methods for reasoning about the behavior of the classification system in presence of a changing world, we propose to reconsider the fairness of the classification task in terms of completeness, consistency, timeliness and reliability, and offer some theoretical results.
\end{abstract}


\maketitle

\section{Introduction}\label{sec1}

Machine Learning (ML) models trained on huge amounts of data are intrinsically biased when dealing with people. Common face recognition systems used in surveillance tasks generate false positives labeling innocent people as suspects. Social credit systems link individuals to the state of their social credit, making decisions based on that score. 
In all of those cases, subjects suffer a credibility deficit due to prejudices related to their social identity \cite{Fricker2007-FRIEIP}: a dark-skinned man could be characterized by a higher risk of recidivism after being arrested; a short-haired skinny young woman -- or a long-haired boy with feminine traits -- might be the target of transphobic attacks following misgendering. Through the deployment of these technologies, society makes the gap separating rich from poor, cisnormative from non-cisnormative individuals, more constitutive as automatized and standardized.

Already before the explosion of ML algorithms, \cite{DBLP:journals/tois/FriedmanN96} offered a framework for understanding three categories of bias in computer systems, assuming the absence of bias as necessary to define their quality. Later on, the emergence of contemporary, data-driven AI systems based on learning has significantly worsened the situation, see e.g. \cite{mehrabi2022survey,DBLP:conf/fat/BuolamwiniG18}. 
%
On this basis, the development and deployment of fairer Artificial Intelligence (AI) systems has been increasingly demanded. 
Such request appears especially relevant in certain application contexts. For example, as examined in \cite{Hanna2021}, face is commonly used as a legitimate mean of gender classification, and this is operazionalized and automatized in technologies such as Automatic Gender Recognition (AGR), which algorithmically derives gender from faces' physical traits to perform classification \cite{Keyes,redu}. This technique relies on the assumption that gender identity can be computationally derived from facial traits. 
However, a recent study \cite{scheuerman} shows that the most famous AGR systems are not able to classify non-binary genders, also performing poorly on transgender individuals. This is due to the fact that AGR incapsulates a binary, cisnormative conception of gender, adopting a male/female scheme which invalidates non-binary identities.

We declare ourselves against the use of gender classification, as considering face as a proxy for detecting gender identity seems to resonate with phrenology 
and physiognomy,
and we believe that the process of automatic gender recognition can easily lead to mismatches between the theoretical understanding of constructs underlying identity and their operationalization \cite{wallach},
especially when it comes to classification of individuals who recognise themselves outside of binarism.
However, we note that this kind of classification is already happening \cite{aliha}, spreading with commercial systems offering gender classification as a standard feature, causing a huge impact on the lives of misgendered individuals. Therefore there are contexts in which it is potentially inevitable that classification exists, and in these contexts it must be fairer. This translates into asking whether there is a strategy to ensure that the labels assigned during classification are as less stereotypical and archetypal as possible. While this paper does not investigates the ethical aspects of AGR, we aim at addressing the issues related to the classification strategies to make them fairer, as an initial study to prepare for implementing mitigation strategies. 

An important task, common to technology and philosophy, is therefore the identification and verification of criteria that may help developing fairness conditions for AI systems. 
While a number of techniques are available to mitigate bias, their primary focus on purely statistical analysis to control accuracy across sensitive classes is clearly insufficient to control social discrimination. A different approach is represented by the explicit formulation of ethical principles to be verified across protected attributes, combining statistical measures with logical reasoning, as formally defined in \cite{DBLP:conf/atal/DAsaroP21, DBLP:conf/aiia/PrimieroD22, DBLP:conf/lori/TerminePD21, dasaro2024checking, KUBYSHKINA2024109212} and implemented by the BRIO tool in \cite{DBLP:conf/beware/CoragliaDGGPPQ23,coraglia2024evaluating}. In this latter context, an important direction to explore for a refined definition of ethically-laden verification criteria is the study of quality dimensions and associated biases. In the following of this paper, we offer a theoretical contribution in this direction, preparing the ground for a future implementation. We argue that, even if maximizing data quality and fairness simultaneously can be hard as improving one can deteriorate the other \cite{AzzaliniCCCDST23}, the task of bias mitigation tools can be supported by reasoning on quality dimensions that so far have been left ignored. In particular, we offer examples to show how dimensions of consistency, completeness, timeliness and reliability can be used to establish fairer AI classification systems. This research is in line with the quest for integrating useful empirical metrics on fairness in AI with asking key (conceptual) questions, see \cite{DBLP:journals/ethicsit/Scantamburlo21}.

The paper is structured as follows. In Section \ref{sec:biases} we offer an overview of fainess definitions and bias types relevant for this work. In Section \ref{sec2} we briefly overview the technical details of a particular bias mitigation tool to illustrate what we consider essential limitations of purely statistical analyses. In Section \ref{sec4} we introduce data quality dimensions arguing for reconsidering their relevance in the task of evaluating the fairness of classification systems, presenting two examples to justify this requirement. In Section \ref{sec5} we propose a definition of fair AI classification that includes such dimensions and formulate some theoretical results. Section \ref{sec6} concludes the work illustrating future research lines.

\section{Fairness and Bias in ML}\label{sec:biases}

Despite a unique definition missing in the literature \cite{DBLP:journals/tois/FriedmanN96,mehrabi2022survey,DBLP:journals/corr/abs-1104-3913,GrgicHlaca2016TheCF,bellamy2018ai,biasmitigationaif,kusner2018counterfactual,DBLP:journals/corr/HardtPS16}, fairness is often presented as corresponding to the avoidance of bias \cite{feuerriegel}. This can be formulated at two distinct levels: first, identifying and correcting problems in datasets \cite{kamiran,10.1145/3552433,weerts2023fairlearn,calmon2017optimized,feldman2015certifying,pmlr-v28-zemel13}, as a model trained with a mislabeled dataset will provide biased outputs; second, correcting the algorithms \cite{GrgicHlaca2016TheCF,kamishima}, as even in the design of algorithms biases can emerge \cite{Movingbeyond}. In the present section we are interested in considering datasets and their labels. Indeed, bias may also affect the label set \cite{sengamedu2023fairlabel,jiang2019identifying}. Accordingly, we talk about \textit{label quality bias} when errors hit the quality of labels. As shown in \cite{northcutt2021pervasive}, the most well-known AI datasets are full of labeling errors. A crucial task is therefore the development of conceptual strategies and technical tools to mitigate bias emergence in both data and label sets.

A variety of approaches and contributions is available in the literature focusing on identifying bias in datasets and labels. Here we list the types of bias which are relevant to the present work, see Table \ref{table:symbols}. Albeit not exhaustive, these lists of biases represent a good starting point to investigate quality dimensions required to address them.
%
We now analyze a common mitigation strategy used by existing tools addressing the issue of bias in data, showing their limitations. We then study the bias in the classification algorithm  (i.e., bias in labels) of the mitigation tool.

\begin{table}
\begin{center}
\begin{minipage}{\textwidth}
\caption{Data and Label Bias.}\label{biasmerged}
\begin{tabular}{>{\raggedright}p{0.3\textwidth} p{0.45\textwidth} p{0.2\textwidth}}  
\toprule
Bias type & Definition & Literature \\
\midrule
\multicolumn{3}{c}{\textbf{Data Bias}} \\
\midrule
\textit{Behavioral bias} & User's behavior can be different across contexts & \cite{olteanu2019social} \\
\textit{Exclusion bias} & Systematic exclusion of some data & \cite{fabbrizzi} \\
\textit{Historical bias} & Cultural prejudices are included into systematic processes & \cite{suresh2021} \\
\textit{Time interval bias} & Data collection in a too limited time range & \cite{nexus} \\
\midrule
\multicolumn{3}{c}{\textbf{Label Bias}} \\
\midrule
\textit{Chronological bias} & Distortion due to temporal changes in the world which data are supposed to represent & \cite{fabbrizzi} \\
\textit{Historical bias} & Cultural prejudices are included into systematic processes & \cite{suresh2021} \\
\textit{Misclassification bias} & Data points are assigned to incorrect categories & \cite{catalogue} \\
\bottomrule
\end{tabular}
\end{minipage}
\end{center}
\end{table}

\section{Mitigating Bias}\label{sec2}

A \textit{bias mitigation algorithm} is a procedure for reducing unwanted bias in training datasets or models, with the aim to improve the fairness metrics. Those algorithms can be classified into three categories \cite{d_Alessandro_2017}: pre-processing, when the training data is modified; in-processing, when the learning algorithm is modified; post-processing, when the predictions are modified.

Several tools are available to audit and  mitigate biases in datasets, thereby attempting to implement diversity and to reach fairness. Among the most common are AIF360 \cite{bellamy2018ai},
Aequitas \cite{saleiro2018aequitas}
and Cleanlab \cite{northcutt2021confidentlearning}.
Recently a post-hoc evaluation model for bias mitigation has been proposed by the tool BRIO \cite{DBLP:conf/beware/CoragliaDGGPPQ23,coraglia2024evaluating}. In this article, we consider Cleanlab as a testbed, illustrating below in Section \ref{sec4} its limitations in view of data quality dimensions.  Instead, we propose a theoretical frame for the resolution of such limitation in Section \ref{sec5}, further illustrating the possibility to implement the present analysis in the tool BRIO. For an overview of the symbols used from now on, see Table \ref{table:symbols}.

\begin{table}[t]
\caption{Symbols used in the present work.}\label{table:symbols}
\begin{center}
\begin{tabular}{>{\raggedright}p{0.25\textwidth} p{0.7\textwidth}} 
\toprule
  $t_n$ & Time index \\ 
  $\mathcal{T}:=\{t_{1}, \dots, t_{n}\}$ & Time frame \\ 
  $d$ & Generic datapoint \\ 
  $i,j,l$ & Data labels \\ 
  $y^*$ & Discrete random variable correctly labeled \\ 
  $\Tilde{y}$ & Discrete random variable wrongly labeled \\ 
  $[m]$ & The set of unique class labels \\ 
  $y^* \to \Tilde{y}$ & A mapping between variables \\ 
  $p_{\mathcal{T}}[(\tilde{y}=i)_{t_{n}}\mid(y^*=j)_{t_{n-m}}]$ & The probability of label $i$ being wrong at time $t_{n}$, given that label $j$ was correct at time $t_{n-m}$ \\ 
  $C_{\tilde{y},y^{*}} [i,j,\mathcal{T}]$ & Temporal confident joint, where the correct label can change from $i$ to $j$ in time frame $\mathcal{T}$ \\ 
  $C_{\tilde{y},y^*} [i,\mathcal{T}]$ & Temporal confident joint, where the correctness of the same fixed label $i$ can change in time frame $\mathcal{T}$ \\ 
  $\varepsilon$ & Change rate \\ 
  $\hat{p'}(\tilde{y};x_i;\theta)$ & Predicted probability of label $\Tilde{y}$ for variable $x_i$ and model parameters $\theta$ \\ 
  $L$ & Label set \\ 
  $X$ & AI system \\ 
  $L_{t_{1}}:=\{l_{1}, \dots,l_{n}\}$ & Partition of the label set \\ 
  $P$ & Population of interest \\ 
  $p$ & An element from $P$ \\ 
  $d(X)_{\mathcal{T}}$ & A datapoint in system $X$ over time frame $\mathcal{T}$ \\ 
  $y^*(d)$ & A correct label for the datapoint $d$ \\ 
  $\pi$ & Threshold variable \\ 
\bottomrule
\end{tabular}
\end{center}
\end{table}

Cleanlab is a 
framework to find label errors in datasets. 
It uses Confident Learning (CL), an approach which focuses on label quality with the aim to address uncertainty in dataset labels using three principles: counting examples that are likely to belong to another class using the confident joint and probabilistic thresholds to find label errors and to estimate noise; pruning noisy data; and ranking examples to train with confidence on clean data. 
The three approaches are combined by an initial assumption of a class-conditional noise process, to directly estimate the joint distribution between noisy given labels and uncorrupted unknown ones. For every class, the algorithm learns the probability of it being mislabeled as any other class. This assumption may have exceptions but it is considered reasonable. For example, a ``cat" is more likely to be mislabeled as ``tiger" than as ``airplane". This assumption is provided by the classification noise process (CNP, \cite{DBLP:journals/ml/AngluinL87}), which leads to the conclusion that the label noise only depends on the latent true class, not on the data. 
CL \cite{northcutt2021confidentlearning} exactly finds label errors in datasets by estimating the joint distribution of noisy and true labels. The idea is that when the predicted probability of an example is greater than a threshold per class, we confidently consider that example as actually belonging to the class of that threshold, where the thresholds for each class are the average predicted probability of examples in that class. Given $\Tilde{y} \in [m]$ takes an observed, noisy label (potentially flipped to an incorrect class); and $y* \in [m]$ takes the unknown (latent), true, uncorrupted label (latent true label), 
CL assumes that for every example it exists a correct label $y*$ and defines a class-conditional noise process mapping $y*\to{\tilde{y}}$, such that every label in class $ j\in[m] $ may be independently mislabeled as class $i\in[m]$, with probability
$p(\tilde{y}=i\mid y*=j)$. So, maps are associations of data to wrong labels. Then CL estimates $p(\tilde{y}\mid y*)$ and $p(y*)$ jointly, evaluating the joint distribution of label noise $p(\tilde{y},y*)$ between noisy given labels and uncorrupted unknown labels. CL aims to estimate every $p(\Tilde{y},y*)$ as a matrix $Q_{\Tilde{y},y*}$ to find all mislabeled examples $x$ in dataset $X$, where $ y*\neq\Tilde{y}$. Given as inputs the out-of-sample predicted probabilities $\hat{P}_{k,i}$ and the vector of noisy labels $\Tilde{y}_k$, the procedure is divided into three steps: estimation of $\hat{Q}_{\Tilde{y},y*}$ to characterize class-conditional label noise, filtering of noisy examples, training with the errors found.

To estimate $\hat{Q}_{\Tilde{y},y*}$ i.e. the joint distribution of noisy labels $\Tilde{y}$ and true labels $y*$, CL counts examples that may belong to another class using a statistical data structure named confident joint $C_{\Tilde{y},y*}$, formally defined as follows

\begin{equation}
C_{\tilde{y},y*} [i][j] :=\mid\hat{X}_{\tilde{y}=i,y*=j}\mid
\label{confidentjoint}
\end{equation}



In other words, the confident joint estimates the set $X_{\tilde{y}=i,y*=j}$ of examples with noisy label \textit{i} which actually have true label \textit{j} by making a partition of the dataset $X$ into bins $\hat{X}_{\tilde{y}=i,y*=j} $, namely the set of examples labeled $\tilde{y}=i $ with \textit{large enough} expected probability $\hat{p}(\tilde{y}=j;x,\theta)$ to belong to class $y*=j$, determined by a per-class threshold $t_j$, where $\theta$ is the model. 



This kind of tools are extremely useful in estimating label error probabilities. However they have some limitations, and it is easy to formulate examples for which their strategy seems unsound. A first problem arises from the initial assumption of the categoricity of data.
Take for example the case of gender labeling of facial images, which is typically binary (i.e. with values male, female). 
For each datapoint, a classification algorithm calculates the projected probability that an image is assigned to the respective label. Consider though two very noisy cases: images of non-binary individuals; images of transgender individuals.
In the former case, the label set becomes incomplete with respect to the dataset; in the second case, the dataset is inconsistent with respect to the label set. Hence, there can be datapoints that have either 1) none of the available labels as the correct one, or 2) at different times they can be under different labels. By definition, if we have disjoint labels there can be high accuracy but only on those datapoints which identify themselves in the disjointed categories.
In situations like these, it appears that the dimension of accuracy alone does no longer satisfy the correctness of the classification algorithm. In terms of quality dimensions, the possibility of an uncategorical datapoint or that of a moving datapoint is no longer only an accuracy problem. Hence, the identification of other data quality dimensions to be implemented in tools for bias mitigation may help achieve more fairness in the classification task. In the next section we suggest an improvement of the classification strategy by adding dimensions that should be considered when evaluating the fairness of the classification itself.

\section{Extending Data Dimensions for Fair AI}\label{sec4}


In the literature, data quality dimensions are defined both informally and qualitatively. Metrics can be associated as indicators of the dimension’s quality. However, there is no single and objective vision of data quality dimensions, nor a universal definition for each dimension. This is because often dimensions escape or exceed a formal definition. The cause of the large amount of dimensions \cite{DBLP:books/sp/dcsa/Batini06,DBLP:journals/jmis/WangS96} also lies in the fact that data aim to represent all spatial, temporal and social phenomena of the real world \cite{stefano}. 
Furthermore, they are constantly evolving in response to continuous development of new data-driven technologies.

For the purposes of our analysis, we focus on the following basic set of data quality dimensions which is the focus of the majority of authors in the literature \cite{batini2009methodologies,scannap2002}:

\begin{itemize}
    \item \textit{Accuracy}, i.e. the closeness between a value $v$ and a value $v'$, where the latter is the correct representation of the real-life phenomenon that $v$ aims to represent \cite{DBLP:books/sp/dcsa/Batini06};
    \item \textit{Completeness}, i.e. the level at which data have the sufficient breadth, depth, and scope for their task \cite{DBLP:journals/jmis/WangS96,DBLP:journals/cacm/PipinoLW02,DBLP:books/sp/dcsa/Batini06};
    \item \textit{Consistency}, i.e. the coherence dimension: it amounts to check whether or not the semantic rules defined on a set of data elements have been respected \cite{DBLP:books/sp/dcsa/Batini06};
    \item \textit{Timeliness}, the data freshness over time for a specific task \cite{rula2014time,rula2016quality}.
\end{itemize}

We thus indicate them as potential candidates to be implemented in the context of bias mitigation strategies. In particular, we argue that, as data are characterized by evolution over time, the timeliness dimension \cite{DBLP:books/sp/dcsa/Batini06} can be taken as basis for other categories of data quality.
We aim at suggesting improvements on errors identification in the classification of datapoints, using the gender attribute as an illustrative case. 
We thus suggest the extension of classification with dimensions of completeness, consistency and timeliness and then return to Cleanlab to illustrate how this extension could be practically implemented.

\subsection{Incomplete Label Set and Inconsistent Labeling}

Consider the first example of a datapoint which represents a non-binary individual. This kind of identity is rarely considered in technology
\cite{patching}.
Non-binary identities do not recognize themselves within the binary approach characteristic of classification systems. As such, individual identity is not correctly recognized by the classification system, highlighting the insufficiency of the model which flattens the gender identity umbrella on the two options of male/female.
The conceptual solution would be to simply assume the label set as incomplete. This means that the bias origin is in the pre-processing phase,
and a possible strategy is to extend the partition of the labels adding categories as appropriate, e.g. “non binary”. The problem is here reduced to the consideration of the completeness of the label set.
\cite{scheuerman} can be considered a first attempt in this direction. 


Consider now a transgender datapoint whose identity shifts over time, being a fluid datapoint by definition.
Currently AI systems operationalize gender in a way which is completely trans-exclusive, see e.g. \cite{redu,Keyes}.
However, identity is not static: it may move with respect to
the labels we have, leading the datapoint to be configured in a label or in a different one during a selected time range. In this case, any extension of the label set is misleading, or at least insufficient. Here we cannot just add more categories, but we have to find a logical solution to changing the label of the same datapoint at different timepoints.

\subsection{Enter Time}

The two problems above can be formulated adding to completeness and consistency the dimension of temporality. 
Thus, an important starting point is represented by adding the dimension of timeliness, which concerns the degree to which data represent reality within a certain defined time range for a given population. 

We suggest here considering the labeling task within a given time frame, whose length depends on the dataset and the classification task over the pairing of datapoints to labels, to measure a probability of a label-change over time. 
Intuitively, if the analysis is performed less than a certain number of timestamps away from the last data labeling, then we consider the labeling still valid. Otherwise, a new analysis with respect to both completeness of the dataset and label set must be performed. Technically, this means associating temporal parameters to labels and to compute the probability that a given label might change over the given time frame. The probability of a label being correct (its accuracy) decreases within the respective temporal window. In particular, reasoning on the temporal evolution of the dataset could allow us to model the evolution of the label partitions. Two fundamental theses are suggested for evaluation: the correctness of the task does no longer assume static completeness of the label set, i.e. given the label set is complete at time $t_n$, it can be incomplete at time $t_{n+m}$; the labeling does no longer assume static perseverance of the labels, that is, given a label $i$ that is correct at a time $t_n$ for a datapoint $d$, it could be incorrect at a later time, and conversely if it is incorrect it could become correct.

\subsection{Back to Cleanlab}

Considering a possible implementation in Cleanlab able to account for such differences implies renouncing the starting assumption on the categoricity of the data. Instead, assume that the probability of assigning a label may change over time. This can be formulated in two distinct ways.
First, the probability value of a given label $i$ being wrong, given a label $j$ is correct (their distance) may change over time. The task is now to give a mapping of all the label-variable pairs, i.e. given a mapping $ y*\to{\Tilde{y}} $ between variables, where $y*$ is the correct label and $\Tilde{y}$ the wrong one, compute the probability over the time frame $\mathcal{T}:=\{t_{1}, \dots, t_{n}\}$

\begin{equation}
    p_{\mathcal{T}}[(\tilde{y}=i)_{t_{n}}\mid(y*=j)_{t_{n-m}}]
\end{equation}

such that label $i$ is wrong at time $t_{n}$, given that label $j$ was correct at time $t_{n-m}$. This probability can increase or decrease, depending on the dataset and on the label set.
For the definition of the confident joint, this means taking the evaluation of all the elements that have an incorrect label $i$ when their correct label is $j$, and then associate the wrong label to a time $t_{n}$ and the correct label to a previous time. This estimate must be made on all time points, so for every $m<n$. Given a timepoint $n$ at which the label is wrong, the estimate on all pairs of probabilities for that point with a previous point in which another label can be correct has to be computed

\begin{equation}
C_{\tilde{y},y*} [i,j,\mathcal{T}] :=\sum^{n\in \mathcal{T}}_{1\leq m<n\in\mathcal{T}}\mid \hat{X}_{\tilde{y}=i_{t_{n}},y*=j_{t_{n-m}}} \mid
\label{confidentjoint1}
\end{equation}

Second, given a mapping $ y*\to{\Tilde{y}} $ between variables, where $y*$ is the correct label and $ \Tilde{y} $ the wrong one, what is the probability 
    
\begin{equation}
p_{\mathcal{T}}[(\tilde{y}=i)_{t_{n}}\mid(y*=i)_{t_{n-m}}]
\end{equation}

such that label $i$ is wrong at time $t_{n}$, given that the same label $i$ was correct at time $t_{n-m}$? In this case, the same label is fixed and the probability that it becomes incorrect can be calculated. The definition of confident joint thus becomes

\begin{equation} C_{\tilde{y},y*} [i,\mathcal{T}] :=\sum^{n\in \mathcal{T}}_{1\leq m<n\in\mathcal{T}}\mid \hat{X}_{\tilde{y}=i_{t_{n}},y*=i_{t_{n-m}}}\mid
\label{confidentjoint2} 
\end{equation}

To illustrate the point we consider a toy example. Compute

\begin{equation}
p(\tilde{y}=i\mid y*=j)=\frac{{p(y*=j\mid \tilde{y}=i)}\cdot{p(\tilde{y}=i)}}{p(y*=j)}=\frac{\frac{[p(y*=j \land \tilde{y}=i)]}{p{(\tilde{y}=i)}}\cdot{p{(\tilde{y}=i)}}}{p(y*=j)}
\label{toyexample}
\end{equation}

i.e. the error rate of $y*=male$ has to be determined. First, a confusion matrix is constructed to analyze errors. Suppose to have a dataset of 10 datapoints, see Figure \ref{fig:t1}.
From the matrix, $ p(y*=j)= 5/10 $ and $ p(\Tilde{y}=i)= 4/10 $. So there are 5 women, of which 2 are incorrectly labeled ``male" and 3 are correctly labeled ``female", and 5 men of which 1 is incorrectly labeled ``female" and 4 are correctly labeled ``male".
Replacing the values in Equation \ref{toyexample}, $p(\tilde{y}=i\mid y*=j)=0.2$.
The obtained value represents the error rate of the ``male" label, i.e. the probability of a male datapoint being labeled ``female". Looking at the diagonals, the true positive rate TPR = 70\% and the false positive rate FPR = 30\%.

\begin{figure}[h]
\centering
\begin{minipage}{0.45\textwidth}
    \centering
    \renewcommand{\arraystretch}{1.5}
    \setlength{\tabcolsep}{1.5mm}
    \begin{tabular}{c c c}
        \multicolumn{1}{c}{} & \multicolumn{2}{c}{\textbf{Actual}} \\[0.5ex]
        \textbf{Predicted} & $y^* = \text{male}$ & $y^* = \text{female}$ \\ \cline{2-3}
        $\tilde{y} = \text{male}$ & 4 & 2 \\ \cline{2-3}
        $\tilde{y} = \text{female}$ & 1 & 3 \\ \cline{2-3}
    \end{tabular}
    \caption{Confusion matrix at time $n$.}
\end{minipage}%
\hfill
\begin{minipage}{0.45\textwidth}
    \centering
    \renewcommand{\arraystretch}{1.5}
    \setlength{\tabcolsep}{1.5mm}
    \begin{tabular}{c c c}
        \multicolumn{1}{c}{} & \multicolumn{2}{c}{\textbf{Actual}} \\[0.5ex]
        \textbf{Predicted} & $y^* = \text{male}$ & $y^* = \text{female}$ \\ \cline{2-3}
        $\tilde{y} = \text{male}$ & 2 & 3 \\ \cline{2-3}
        $\tilde{y} = \text{female}$ & 3 & 2 \\ \cline{2-3}
    \end{tabular}
    \caption{Confusion matrix at time $n+m$.}
\end{minipage}
\end{figure}


Consider now the same dataset at a later time $t_{n+m}$. The labels might have changed. From the matrix, $ p(y*=j)= 5/10 $ and that $ p(\tilde{y}=i)= 5/10 $. Now there are 5 women, of which 3 are incorrectly labeled ``male" and 2 are correctly labeled ``female", and 5 men of which 3 are incorrectly labeled``female" and 2 are correctly labeled ``male". Replacing again the values in \ref{toyexample}, $p'(\tilde{y}=i\mid y*=j)=0.6$.
In this case the true positive rate TPR = 40\% and the false positive rate FPR = 60\%.

To understand how the error rate changes, the difference between the two matrices has to be considered.
Thus, the change rate can be computed as $\varepsilon=\hat{p'}(\tilde{y};x_i;\theta)-\hat{p}(\tilde{y};x_i;\theta)=0.4 $.

Now $ p_{\mathcal{T}}[(\tilde{y}=i)_{t_{n}}\mid(y*=j)_{t_{n-m}}] $
can be written as $ p_{\mathcal{T}}[(\tilde{y}=i)_{t_{n+m}}\mid(y*=j)_{t_{n}}] $. Thus, at a time $ t_{n} $ we have $ p_{t_{n}}(y*=j) = 1 - p(\tilde{y}=i)_{t_{n}}) $. At a subsequent time $ t_{n+m} $ we have $ p_{t_{n+m}}(y*=j) = 1 - p(\tilde{y}=i)_{t_{n+m}}$. Equation \ref{toyexample} can be computed with respect to time as

\begin{equation}
p_{\mathcal{T}}[(\tilde{y}=i)_{t_{n+m}}\mid(y*=j)_{t_{n}}] = \frac{{p[(y*=j)\mid(\tilde{y}=i)]_{t_{n+m}}}\cdot{[p(\tilde{y}=i){t_{n}}\pm\varepsilon]}}{p(y*=j)_{t_{n}}}=0.288
\label{toytime}
\end{equation}

This value represents the (highest) probability that a given label is wrong at a given time, provided it was correct at some previous time. Indirectly, this also expresses the probability that the labeling set is applied to a dataset containing a point for which the labeling becomes inconsistent over time.

\section{Temporal-based Fairness in AI}\label{sec5}

We have argued that a more general discussion on the data dimensions to be adopted in bias mitigation tools is needed, and in particular that the dimension of timeliness is crucial. In this section we summarise our proposal and offer non-exhaustive criteria for fairness in AI based on such temporal approach along with some basic theoretical results.

The first metric that has been addressed in this work is completeness as applied to the label set. In a world where gender classification is actually changing, the present strategy includes the completeness dimension in the quality assessment, verifying that the label set is complete with respect to the ontology of the world at the time this assessment is made. The solution here is to extend the label set as desired adding new labels for the classification task, as already suggested in \cite{scheuerman}. Additionally, we suggest an explicit temporal parametrization: completeness can be considered as a relationship between a label set and an individual $p$ belonging to a certain population $P$, where $p$ is any domain item that enters $P$ at a time $t$. We must ensure that a correct label $l$ exists for each datapoint in the dataset at each time.

\begin{definition}[Completeness of a label set]\label{def:compl}
A label set $L$ for a classification algorithm in a AI system $X$ is considered complete over a time frame $\mathcal{T}:\{t_{1},\dots, t_{n}\}$ denoted as $Compl_{\mathcal{T}}(L(X))$ iff given two partitions $ L_{t_{1}}:=\{l_{1}, \dots,l_{n}\} $ and $ L_{t_{n}}:=\{l'_{1}, \dots, l'_{n}\} $, where possibly $L_{t_{1}} \cap L_{t_{n}} \neq \emptyset$ for all $ (p \in P)_{\mathcal{T}}$ s.t. $p \in d(X)_{\mathcal{T}}$ there is $l \in L_{t_{1}} \cup L_{t_{n}} $ s.t. $y^{*}(d)=l $.
\end{definition}

In other words, the completeness of a dataset over a time frame is granted if for every datapoint representing an element in the population of interest there exists at any two possibly consecutive points in time a correct label for it.

Next, we considered consistency of the label set with respect to datapoints possibly shifting in categorization. The method here again is to reduce consistency to timeliness. We suggest to compute the probability of an inconsistency arising from a correct label change.
Accuracy, albeit the most used metric for evaluating classification models’ performances due to its easy calculability and interpretation, is reductive, trivial and incorrect in some contexts. For example, if the distribution of the class is distorted, accuracy is no longer a useful, nor a relevant metric. Even worse, sometimes greater accuracy leads to greater unfairness \cite{practical}:  some labels like race or gender may allow models to be more predictive, although it seems to be often controversial to use such categories to increase predictive performance. 
We have suggested to consider temporal accuracy \cite{BaSc} as a function of the error rate over time.

The ability to compute the variance in the error rate across time is functional to determine the reliability of AI systems. This metric is linked to the notion of accuracy, as it is considered as a measure of data correctness, see \cite{DBLP:books/sp/dcsa/Batini06}. In \cite{DBLP:journals/jmis/WangS96} and \cite{BaSc} reliability is even contained in the definition of accuracy itself: data must be reliable to satisfy the accuracy dimension. Overall, it seems that reliability is not actually controlled beyond physical reliability, as in the literature on data quality there is no formal definition to compute it. However, following \cite{DAMA} the previously provided temporal approach is again useful: evaluating reliability is based on the revisions which show how close the initial estimate of accuracy is to the following ones. In this sense, reliability can be reduced to accuracy over time in terms of a threshold on the error rate:

\begin{definition}[Reliability of a classification algorithm]\label{def:rel}
A classification algorithm in a AI system $X$ is considered reliable over a time frame $ \mathcal{T}:=\{t_1, \dots, t_n\} $ denoted as $ Rel_{\mathcal{T}}(X) $ iff $ \varepsilon_{\mathcal{T}}(X)< \pi $, for some safe value $\pi$.
\end{definition}

The change rate $\varepsilon$ we have computed shows how much the system's accuracy deteriorates. If it exceeds a fixed safe value $\pi$, the system is no longer accurate. Plain accuracy is the numerical measure at some time $ t \in \mathcal{T}:=\{t_1, \dots, t_n\}$. If this value does not deteriorate over a certain fixed threshold, the system is considered reliable, and therefore accurate with respect to time.

The two previous definitions offer non-exhaustive criteria for the identification of fair AI systems:

\begin{definition}[Fairness for AI classification systems]\label{def:fair}
$Fair_{\mathcal{T}}(X) $ only if $ Rel_{\mathcal{T}}(X) $ and $ Compl_{\mathcal{T}}(L(X))$.
\end{definition}

Hence we claim that fairness requires the system's ability to give reliable and correct outcomes over time.
While we do not consider these properties sufficient, we believe they are necessary.
On this basis, we can formulate two immediate theoretical results:

\begin{theorem}
Given a label set $L$ complete at time $t$, a classification algorithm guarantees a fair classification at time $t'>t$ if and only if the change rate determined with respect to $L$ is $\epsilon < \pi$.
\end{theorem}
\begin{proof}
Assume $ Compl_{t}(L(X))$, then for $Fair_{t'}(X)$ we need to show $Rel_{t'}(X)$ for $t'>t\in \mathcal{T}$. Assume $\epsilon > \pi$, then by Definition \ref{def:rel} reliability is not satisfied; hence, if $ Rel_{\mathcal{T}}(X) $, it must be the case that $\epsilon < \pi$.
\end{proof}

\begin{theorem}
Given a fixed change rate $\epsilon < \pi$, a classification algorithm with fair behaviour at time $t$ remains fair at time $t'>t$ if and only if the change to make the label set complete at time $t'$ does not exceed an $\epsilon'$ such that $\epsilon+\epsilon'>\pi$.
\end{theorem}
\begin{proof}
Consider $Fair_{t}(X)$ with change rate $0 < \pi$ as a base case, then by Definition \ref{def:fair} $ Rel_{t}(X) $ and $Compl_{t}(L(X))$. Now consider $t'>t$ and a required change $\epsilon'$ in $Compl_{t}(L(X))$ such that $Rel_{t'}(X)$ holds. This obviously holds only if $0+\epsilon'<\pi$. Generalize for any $\epsilon>0$.
\end{proof}
Note that in these results the value of $\epsilon$, respectively $\epsilon'$, is a proxy for how much the world has changed at $t'$ with respect to $Compl_{t}(L(X))$.

In the context of an incomplete label set, a detected label bias can originate from an exclusion bias in data, which can also result from a time interval bias. In the case of label-changing datapoints a chronological bias occurs. Then, misclassification bias can be reduced to the two previous types. In the context of use, emergent bias can arise as a result of changes in societies and cultures. It might appear in data as chronological, historical or behavioral bias. Here, a different value bias occurs for example when the users are different from the assumed ones during the system's development. This is the case of ontology switching, to which a label set must adapt. 
These types of bias can be mitigated by implementing the proposed framework. The tool BRIO \cite{DBLP:conf/beware/CoragliaDGGPPQ23,coraglia2024evaluating} works as a post-hoc model evaluation, taking in input the test dataset of the model under investigation and its output. The tool allows to investigate behavioural differences of the model both with respect to an internal analysis on the classes of interest, and externally with respect to chosen reference metrics. Morever, it allows to measure bias amplification comparing the bias present in the dataset and how that manifests itself in the output. 
While the present work does not aim at offering a full implementation of our theoretical analysis for the BRIO tool, some remarks are appropriate. The time-based analysis of completeness and reliability offered in Definitions \ref{def:compl} and \ref{def:rel}, in turn grounding a notion of fairness in Definition \ref{def:fair} are easily implementable in BRIO: both completeness and reliability require the definition of a timeframe to check respectively that any given datapoint of interest is matched against a desirable label and that the overall change rate of error for one or more classes of interest does not surpass a certain threshold. Both features rely on the user for the identification of the desirable label for any datapoint and for the admissible distance.

\section{Conclusion}\label{sec6}
We presented some recommendations for AI systems design, focusing on timeliness as a founding dimension for developing fairer and more inclusive classification tools. Despite the crucial importance of accuracy as shown by significant works such as \cite{DBLP:conf/fat/BuolamwiniG18} and \cite{propublica}, the problem of unfairness in AI systems is much broader and more foundational. This can be expressed in terms of data quality: AI systems are limited in that they maximize accuracy, and even if systems become statistically accurate some problems remain unsolved. This is exemplified by the case of binary gender labeling, which leads to inaccurate simplistic classifications \cite{Edenberg2023-EDEAEL}.
Furthermore, as the work of classification is always a reflection of culture, the completeness of the label set and the (constrained) consistency of labeling have an epistemological value: constructing AIs requires us to understand society, and society reflects an ontology of individuals. For this reason, misgendering is first of all an ontological error \cite{Keyes}.

We suggested that timeliness is a crucial dimension for the definition of gender identity. If we are ready to consider gender as a property that shifts over time \cite{ruberg}, and which can also be declined in the plural, as an individual may identify under more than one - not mutually exclusive - labels, then a change of paradigm is required. 
Design limitations such as binarism and staticity invalidate identities which do not fit into this paradigm. They must be addressed if fairer classifications and more inclusive models of gender are to be designed.

Further work in this direction includes: an implementation and empirical validation of the proposed model through the BRIO tool; 
and the design of an extension to compute the probability of incorrect labels becoming correct over time, i.e. the dual case of what presently addressed.

\begin{acknowledgments}
    This research has been partially funded by the Projects: PRIN2020 BRIO (2020SSKZ7R), PRIN2022 SMARTEST (20223E8Y4X),  “Departments of Excellence 2023-2027” of the Department of Philosophy “Piero Martinetti” of the University of Milan, all awarded by the Italian Ministry of University and Research (MUR); and MUSA – Multilayered Urban Sustainability Action, funded by the European Union – NextGenerationEU, under the National Recovery and Resilience Plan (NRRP) Mission 4 Component 2 Investment Line 1.5: Strenghtening of research structures and creation of R\&D ``innovation ecosystems”, set up of ``territorial leaders in R\&D”.
\end{acknowledgments}

\bibliography{sample-ceur}

\begin{thebibliography}{61}
\expandafter\ifx\csname natexlab\endcsname\relax\def\natexlab#1{#1}\fi
\providecommand{\url}[1]{\texttt{#1}}
\providecommand{\href}[2]{#2}
\providecommand{\path}[1]{#1}
\providecommand{\DOIprefix}{doi:}
\providecommand{\ArXivprefix}{arXiv:}
\providecommand{\URLprefix}{URL: }
\providecommand{\Pubmedprefix}{pmid:}
\providecommand{\doi}[1]{\href{http://dx.doi.org/#1}{\path{#1}}}
\providecommand{\Pubmed}[1]{\href{pmid:#1}{\path{#1}}}
\providecommand{\bibinfo}[2]{#2}
\ifx\xfnm\relax \def\xfnm[#1]{\unskip,\space#1}\fi
\bibitem[{Fricker(2007)}]{Fricker2007-FRIEIP}
\bibinfo{author}{M.~Fricker}, \bibinfo{title}{Epistemic Injustice: Power and the Ethics of Knowing}, \bibinfo{publisher}{New York: Oxford University Press}, \bibinfo{year}{2007}.
\bibitem[{Friedman and Nissenbaum(1996)}]{DBLP:journals/tois/FriedmanN96}
\bibinfo{author}{B.~Friedman}, \bibinfo{author}{H.~Nissenbaum},
\newblock \bibinfo{title}{Bias in computer systems},
\newblock \bibinfo{journal}{{ACM} Trans. Inf. Syst.} \bibinfo{volume}{14} (\bibinfo{year}{1996}) \bibinfo{pages}{330--347}.
\bibitem[{Mehrabi et~al.(2019)Mehrabi, Morstatter, Saxena, Lerman, and Galstyan}]{mehrabi2022survey}
\bibinfo{author}{N.~Mehrabi}, \bibinfo{author}{F.~Morstatter}, \bibinfo{author}{N.~Saxena}, \bibinfo{author}{K.~Lerman}, \bibinfo{author}{A.~Galstyan},
\newblock \bibinfo{title}{A survey on bias and fairness in machine learning},
\newblock \bibinfo{journal}{CoRR} \bibinfo{volume}{abs/1908.09635} (\bibinfo{year}{2019}). \URLprefix \url{http://arxiv.org/abs/1908.09635}. \href{http://arxiv.org/abs/1908.09635}{{\tt arXiv:1908.09635}}.
\bibitem[{Buolamwini and Gebru(2018)}]{DBLP:conf/fat/BuolamwiniG18}
\bibinfo{author}{J.~Buolamwini}, \bibinfo{author}{T.~Gebru},
\newblock \bibinfo{title}{Gender shades: Intersectional accuracy disparities in commercial gender classification},
\newblock in: \bibinfo{booktitle}{Conference on Fairness, Accountability and Transparency, {FAT} 2018, 23-24 February 2018, New York, NY, {USA}}, volume~\bibinfo{volume}{81}, \bibinfo{publisher}{{PMLR}}, \bibinfo{year}{2018}, pp. \bibinfo{pages}{77--91}.
\bibitem[{Hanna et~al.(2021)Hanna, Pape, and Scheuerman}]{Hanna2021}
\bibinfo{author}{A.~Hanna}, \bibinfo{author}{M.~Pape}, \bibinfo{author}{M.~K. Scheuerman},
\newblock \bibinfo{title}{Auto-essentialization: Gender in automated facial analysis as extended colonial project},
\newblock \bibinfo{journal}{Big Data and Society} \bibinfo{volume}{8} (\bibinfo{year}{2021}). \DOIprefix\doi{10.1177/20539517211053712}.
\bibitem[{Keyes(2018)}]{Keyes}
\bibinfo{author}{O.~Keyes},
\newblock \bibinfo{title}{The misgendering machines: Trans/{HCI} implications of automatic gender recognition},
\newblock \bibinfo{journal}{Proc. ACM Hum.-Comput. Interact.} \bibinfo{volume}{2} (\bibinfo{year}{2018}). \URLprefix \url{https://doi.org/10.1145/3274357}. \DOIprefix\doi{10.1145/3274357}.
\bibitem[{Hamidi et~al.(2018)Hamidi, Scheuerman, and Branham}]{redu}
\bibinfo{author}{F.~Hamidi}, \bibinfo{author}{M.~K. Scheuerman}, \bibinfo{author}{S.~M. Branham},
\newblock \bibinfo{title}{Gender recognition or gender reductionism? the social implications of embedded gender recognition systems},
\newblock in: \bibinfo{booktitle}{Proceedings of the 2018 CHI Conference on Human Factors in Computing Systems}, CHI '18, \bibinfo{publisher}{Association for Computing Machinery}, \bibinfo{address}{New York, NY, USA}, \bibinfo{year}{2018}, p. \bibinfo{pages}{1–13}. \URLprefix \url{https://doi.org/10.1145/3173574.3173582}. \DOIprefix\doi{10.1145/3173574.3173582}.
\bibitem[{Scheuerman et~al.(2019)Scheuerman, Paul, and Brubaker}]{scheuerman}
\bibinfo{author}{M.~K. Scheuerman}, \bibinfo{author}{J.~Paul}, \bibinfo{author}{J.~Brubaker},
\newblock \bibinfo{title}{How computers see gender: An evaluation of gender classification in commercial facial analysis services},
\newblock \bibinfo{journal}{Proceedings of the ACM on Human-Computer Interaction} \bibinfo{volume}{3} (\bibinfo{year}{2019}) \bibinfo{pages}{1--33}. \DOIprefix\doi{10.1145/3359246}.
\bibitem[{Jacobs and Wallach(2019)}]{wallach}
\bibinfo{author}{A.~Z. Jacobs}, \bibinfo{author}{H.~M. Wallach},
\newblock \bibinfo{title}{Measurement and fairness},
\newblock \bibinfo{journal}{CoRR} \bibinfo{volume}{abs/1912.05511} (\bibinfo{year}{2019}). \URLprefix \url{http://arxiv.org/abs/1912.05511}. \href{http://arxiv.org/abs/1912.05511}{{\tt arXiv:1912.05511}}.
\bibitem[{Ramon et~al.(2024)Ramon, Olaoye, and Luz}]{aliha}
\bibinfo{author}{A.~Ramon}, \bibinfo{author}{G.~Olaoye}, \bibinfo{author}{A.~Luz},
\newblock \bibinfo{title}{Machine learning algorithms for gender prediction}  (\bibinfo{year}{2024}).
\bibitem[{D'Asaro and Primiero(2021)}]{DBLP:conf/atal/DAsaroP21}
\bibinfo{author}{F.~A. D'Asaro}, \bibinfo{author}{G.~Primiero},
\newblock \bibinfo{title}{Probabilistic typed natural deduction for trustworthy computations},
\newblock in: \bibinfo{editor}{D.~Wang}, \bibinfo{editor}{R.~Falcone}, \bibinfo{editor}{J.~Zhang} (Eds.), \bibinfo{booktitle}{Proceedings of the 22nd International Workshop on Trust in Agent Societies {(TRUST} 2021) Co-located with the 20th International Conferences on Autonomous Agents and Multiagent Systems {(AAMAS} 2021), London, UK, May 3-7, 2021}, volume \bibinfo{volume}{3022} of \textit{\bibinfo{series}{{CEUR} Workshop Proceedings}}, \bibinfo{publisher}{CEUR-WS.org}, \bibinfo{year}{2021}. \URLprefix \url{http://ceur-ws.org/Vol-3022/paper3.pdf}.
\bibitem[{Primiero and D'Asaro(2022)}]{DBLP:conf/aiia/PrimieroD22}
\bibinfo{author}{G.~Primiero}, \bibinfo{author}{F.~A. D'Asaro},
\newblock \bibinfo{title}{Proof-checking bias in labeling methods},
\newblock in: \bibinfo{editor}{G.~Boella}, \bibinfo{editor}{F.~A. D'Asaro}, \bibinfo{editor}{A.~Dyoub}, \bibinfo{editor}{G.~Primiero} (Eds.), \bibinfo{booktitle}{Proceedings of 1st Workshop on Bias, Ethical AI, Explainability and the Role of Logic and Logic Programming {(BEWARE} 2022) co-located with the 21th International Conference of the Italian Association for Artificial Intelligence (AI*IA 2022), Udine, Italy, December 2, 2022}, volume \bibinfo{volume}{3319} of \textit{\bibinfo{series}{{CEUR} Workshop Proceedings}}, \bibinfo{publisher}{CEUR-WS.org}, \bibinfo{year}{2022}, pp. \bibinfo{pages}{9--19}. \URLprefix \url{https://ceur-ws.org/Vol-3319/paper1.pdf}.
\bibitem[{Termine et~al.(2021)Termine, Primiero, and D'Asaro}]{DBLP:conf/lori/TerminePD21}
\bibinfo{author}{A.~Termine}, \bibinfo{author}{G.~Primiero}, \bibinfo{author}{F.~A. D'Asaro},
\newblock \bibinfo{title}{Modelling accuracy and trustworthiness of explaining agents},
\newblock in: \bibinfo{editor}{S.~Ghosh}, \bibinfo{editor}{T.~Icard} (Eds.), \bibinfo{booktitle}{Logic, Rationality, and Interaction - 8th International Workshop, {LORI} 2021, Xi'ian, China, October 16-18, 2021, Proceedings}, volume \bibinfo{volume}{13039} of \textit{\bibinfo{series}{Lecture Notes in Computer Science}}, \bibinfo{publisher}{Springer}, \bibinfo{year}{2021}, pp. \bibinfo{pages}{232--245}. \URLprefix \url{https://doi.org/10.1007/978-3-030-88708-7\_19}. \DOIprefix\doi{10.1007/978-3-030-88708-7\_19}.
\bibitem[{D'Asaro et~al.(2024)D'Asaro, Genco, and Primiero}]{dasaro2024checking}
\bibinfo{author}{F.~A. D'Asaro}, \bibinfo{author}{F.~Genco}, \bibinfo{author}{G.~Primiero}, \bibinfo{title}{Checking trustworthiness of probabilistic computations in a typed natural deduction system}, \bibinfo{year}{2024}. \href{http://arxiv.org/abs/2206.12934}{{\tt arXiv:2206.12934}}.
\bibitem[{Kubyshkina and Primiero(2024)}]{KUBYSHKINA2024109212}
\bibinfo{author}{E.~Kubyshkina}, \bibinfo{author}{G.~Primiero},
\newblock \bibinfo{title}{A possible worlds semantics for trustworthy non-deterministic computations},
\newblock \bibinfo{journal}{International Journal of Approximate Reasoning}  (\bibinfo{year}{2024}) \bibinfo{pages}{109212}. \URLprefix \url{https://www.sciencedirect.com/science/article/pii/S0888613X24000999}. \DOIprefix\doi{https://doi.org/10.1016/j.ijar.2024.109212}.
\bibitem[{Coraglia et~al.(2023)Coraglia, D'Asaro, Genco, Giannuzzi, Posillipo, Primiero, and Quaggio}]{DBLP:conf/beware/CoragliaDGGPPQ23}
\bibinfo{author}{G.~Coraglia}, \bibinfo{author}{F.~A. D'Asaro}, \bibinfo{author}{F.~A. Genco}, \bibinfo{author}{D.~Giannuzzi}, \bibinfo{author}{D.~Posillipo}, \bibinfo{author}{G.~Primiero}, \bibinfo{author}{C.~Quaggio},
\newblock \bibinfo{title}{Brioxalkemy: a bias detecting tool},
\newblock in: \bibinfo{editor}{G.~Boella}, \bibinfo{editor}{F.~A. D'Asaro}, \bibinfo{editor}{A.~Dyoub}, \bibinfo{editor}{L.~Gorrieri}, \bibinfo{editor}{F.~A. Lisi}, \bibinfo{editor}{C.~Manganini}, \bibinfo{editor}{G.~Primiero} (Eds.), \bibinfo{booktitle}{Proceedings of the 2nd Workshop on Bias, Ethical AI, Explainability and the role of Logic and Logic Programming co-located with the 22nd International Conference of the Italian Association for Artificial Intelligence (AI*IA 2023), Rome, Italy, November 6, 2023}, volume \bibinfo{volume}{3615} of \textit{\bibinfo{series}{{CEUR} Workshop Proceedings}}, \bibinfo{publisher}{CEUR-WS.org}, \bibinfo{year}{2023}, pp. \bibinfo{pages}{44--60}. \URLprefix \url{https://ceur-ws.org/Vol-3615/paper4.pdf}.
\bibitem[{Coraglia et~al.(2024)Coraglia, Genco, Piantadosi, Bagli, Giuffrida, Posillipo, and Primiero}]{coraglia2024evaluating}
\bibinfo{author}{G.~Coraglia}, \bibinfo{author}{F.~A. Genco}, \bibinfo{author}{P.~Piantadosi}, \bibinfo{author}{E.~Bagli}, \bibinfo{author}{P.~Giuffrida}, \bibinfo{author}{D.~Posillipo}, \bibinfo{author}{G.~Primiero}, \bibinfo{title}{Evaluating ai fairness in credit scoring with the brio tool}, \bibinfo{year}{2024}. \href{http://arxiv.org/abs/2406.03292}{{\tt arXiv:2406.03292}}.
\bibitem[{Azzalini et~al.(2023)Azzalini, Cappiello, Criscuolo, Cuzzucoli, Dangelo, Sancricca, and Tanca}]{AzzaliniCCCDST23}
\bibinfo{author}{F.~Azzalini}, \bibinfo{author}{C.~Cappiello}, \bibinfo{author}{C.~Criscuolo}, \bibinfo{author}{S.~Cuzzucoli}, \bibinfo{author}{A.~Dangelo}, \bibinfo{author}{C.~Sancricca}, \bibinfo{author}{L.~Tanca},
\newblock \bibinfo{title}{Data quality and fairness: Rivals or friends?},
\newblock in: \bibinfo{editor}{D.~Calvanese}, \bibinfo{editor}{C.~Diamantini}, \bibinfo{editor}{G.~Faggioli}, \bibinfo{editor}{N.~F. 0001}, \bibinfo{editor}{S.~M. 0001}, \bibinfo{editor}{G.~Silvello}, \bibinfo{editor}{L.~Tanca} (Eds.), \bibinfo{booktitle}{Proceedings of the 31st Symposium of Advanced Database Systems, Galzingano Terme, Italy, July 2nd to 5th, 2023}, volume \bibinfo{volume}{3478} of \textit{\bibinfo{series}{CEUR Workshop Proceedings}}, \bibinfo{publisher}{CEUR-WS.org}, \bibinfo{year}{2023}, pp. \bibinfo{pages}{239--247}. \URLprefix \url{https://ceur-ws.org/Vol-3478/paper68.pdf}.
\bibitem[{Scantamburlo(2021)}]{DBLP:journals/ethicsit/Scantamburlo21}
\bibinfo{author}{T.~Scantamburlo},
\newblock \bibinfo{title}{Non-empirical problems in fair machine learning},
\newblock \bibinfo{journal}{Ethics Inf. Technol.} \bibinfo{volume}{23} (\bibinfo{year}{2021}) \bibinfo{pages}{703--712}. \URLprefix \url{https://doi.org/10.1007/s10676-021-09608-9}. \DOIprefix\doi{10.1007/s10676-021-09608-9}.
\bibitem[{Dwork et~al.(2011)Dwork, Hardt, Pitassi, Reingold, and Zemel}]{DBLP:journals/corr/abs-1104-3913}
\bibinfo{author}{C.~Dwork}, \bibinfo{author}{M.~Hardt}, \bibinfo{author}{T.~Pitassi}, \bibinfo{author}{O.~Reingold}, \bibinfo{author}{R.~S. Zemel},
\newblock \bibinfo{title}{Fairness through awareness},
\newblock \bibinfo{journal}{CoRR} \bibinfo{volume}{abs/1104.3913} (\bibinfo{year}{2011}). \URLprefix \url{http://arxiv.org/abs/1104.3913}. \href{http://arxiv.org/abs/1104.3913}{{\tt arXiv:1104.3913}}.
\bibitem[{Grgic-Hlaca et~al.(2016)Grgic-Hlaca, Zafar, Gummadi, and Weller}]{GrgicHlaca2016TheCF}
\bibinfo{author}{N.~Grgic-Hlaca}, \bibinfo{author}{M.~B. Zafar}, \bibinfo{author}{K.~P. Gummadi}, \bibinfo{author}{A.~Weller},
\newblock \bibinfo{title}{The case for process fairness in learning: Feature selection for fair decision making},
\newblock \bibinfo{year}{2016}.
\bibitem[{Bellamy et~al.(2018)Bellamy, Dey, Hind, Hoffman, Houde, Kannan, Lohia, Martino, Mojsilovic, Nagar, Ramamurthy, Richards, Saha, Sattigeri, Singh, Varshney, and Zhang}]{bellamy2018ai}
\bibinfo{author}{R.~K.~E. Bellamy}, \bibinfo{author}{K.~Dey}, \bibinfo{author}{M.~Hind}, \bibinfo{author}{S.~C. Hoffman}, \bibinfo{author}{S.~Houde}, \bibinfo{author}{K.~Kannan}, \bibinfo{author}{P.~Lohia}, \bibinfo{author}{S.~Martino, J. and.~Mehta}, \bibinfo{author}{A.~Mojsilovic}, \bibinfo{author}{S.~Nagar}, \bibinfo{author}{K.~N. Ramamurthy}, \bibinfo{author}{J.~Richards}, \bibinfo{author}{D.~Saha}, \bibinfo{author}{P.~Sattigeri}, \bibinfo{author}{M.~Singh}, \bibinfo{author}{K.~R. Varshney}, \bibinfo{author}{Y.~Zhang}, \bibinfo{title}{Ai fairness 360: An extensible toolkit for detecting, understanding, and mitigating unwanted algorithmic bias}, \bibinfo{year}{2018}.
\bibitem[{Aasheim et~al.(2020)Aasheim, Hufthammer, Ånneland, Brynjulfsen, and Slavkovik}]{biasmitigationaif}
\bibinfo{author}{T.~H. Aasheim}, \bibinfo{author}{K.~Hufthammer}, \bibinfo{author}{S.~Ånneland}, \bibinfo{author}{H.~Brynjulfsen}, \bibinfo{author}{M.~Slavkovik},
\newblock \bibinfo{title}{Bias mitigation with aif360: A comparative study},
\newblock in: \bibinfo{booktitle}{Proceedings of the NIK-2020 Conference}, \bibinfo{year}{2020}. \URLprefix \url{http://www.nik.no/}.
\bibitem[{Kusner et~al.(2018)Kusner, Loftus, Russell, and Silva}]{kusner2018counterfactual}
\bibinfo{author}{M.~J. Kusner}, \bibinfo{author}{J.~R. Loftus}, \bibinfo{author}{C.~Russell}, \bibinfo{author}{R.~Silva}, \bibinfo{title}{Counterfactual fairness}, \bibinfo{year}{2018}. \href{http://arxiv.org/abs/1703.06856}{{\tt arXiv:1703.06856}}.
\bibitem[{Hardt et~al.(2016)Hardt, Price, and Srebro}]{DBLP:journals/corr/HardtPS16}
\bibinfo{author}{M.~Hardt}, \bibinfo{author}{E.~Price}, \bibinfo{author}{N.~Srebro},
\newblock \bibinfo{title}{Equality of opportunity in supervised learning},
\newblock \bibinfo{journal}{CoRR} \bibinfo{volume}{abs/1610.02413} (\bibinfo{year}{2016}). \URLprefix \url{http://arxiv.org/abs/1610.02413}. \href{http://arxiv.org/abs/1610.02413}{{\tt arXiv:1610.02413}}.
\bibitem[{Feuerriegel et~al.(2020)Feuerriegel, Dolata, and Schwabe}]{feuerriegel}
\bibinfo{author}{S.~Feuerriegel}, \bibinfo{author}{M.~Dolata}, \bibinfo{author}{G.~Schwabe},
\newblock \bibinfo{title}{Fair {AI}: Challenges and opportunities},
\newblock \bibinfo{journal}{Business \& Information Systems Engineering} \bibinfo{volume}{62} (\bibinfo{year}{2020}). \DOIprefix\doi{10.1007/s12599-020-00650-3}.
\bibitem[{Kamiran and Calders(2011)}]{kamiran}
\bibinfo{author}{F.~Kamiran}, \bibinfo{author}{T.~Calders},
\newblock \bibinfo{title}{Data pre-processing techniques for classification without discrimination},
\newblock \bibinfo{journal}{Knowledge and Information Systems} \bibinfo{volume}{33} (\bibinfo{year}{2011}). \DOIprefix\doi{10.1007/s10115-011-0463-8}.
\bibitem[{Azzalini et~al.(2022)Azzalini, Criscuolo, and Tanca}]{10.1145/3552433}
\bibinfo{author}{F.~Azzalini}, \bibinfo{author}{C.~Criscuolo}, \bibinfo{author}{L.~Tanca},
\newblock \bibinfo{title}{E-fair-db: Functional dependencies to discover data bias and enhance data equity},
\newblock \bibinfo{journal}{J. Data and Information Quality} \bibinfo{volume}{14} (\bibinfo{year}{2022}). \URLprefix \url{https://doi.org/10.1145/3552433}. \DOIprefix\doi{10.1145/3552433}.
\bibitem[{Weerts et~al.(2023)Weerts, Dudík, Edgar, Jalali, Lutz, and Madaio}]{weerts2023fairlearn}
\bibinfo{author}{H.~Weerts}, \bibinfo{author}{M.~Dudík}, \bibinfo{author}{R.~Edgar}, \bibinfo{author}{A.~Jalali}, \bibinfo{author}{R.~Lutz}, \bibinfo{author}{M.~Madaio}, \bibinfo{title}{Fairlearn: Assessing and improving fairness of ai systems}, \bibinfo{year}{2023}. \href{http://arxiv.org/abs/2303.16626}{{\tt arXiv:2303.16626}}.
\bibitem[{Calmon et~al.(2017)Calmon, Wei, Ramamurthy, and Varshney}]{calmon2017optimized}
\bibinfo{author}{F.~P. Calmon}, \bibinfo{author}{D.~Wei}, \bibinfo{author}{K.~N. Ramamurthy}, \bibinfo{author}{K.~R. Varshney}, \bibinfo{title}{Optimized data pre-processing for discrimination prevention}, \bibinfo{year}{2017}. \href{http://arxiv.org/abs/1704.03354}{{\tt arXiv:1704.03354}}.
\bibitem[{Feldman et~al.(2015)Feldman, Friedler, Moeller, Scheidegger, and Venkatasubramanian}]{feldman2015certifying}
\bibinfo{author}{M.~Feldman}, \bibinfo{author}{S.~Friedler}, \bibinfo{author}{J.~Moeller}, \bibinfo{author}{C.~Scheidegger}, \bibinfo{author}{S.~Venkatasubramanian}, \bibinfo{title}{Certifying and removing disparate impact}, \bibinfo{year}{2015}. \href{http://arxiv.org/abs/1412.3756}{{\tt arXiv:1412.3756}}.
\bibitem[{Zemel et~al.(2013)Zemel, Wu, Swersky, Pitassi, and Dwork}]{pmlr-v28-zemel13}
\bibinfo{author}{R.~Zemel}, \bibinfo{author}{Y.~Wu}, \bibinfo{author}{K.~Swersky}, \bibinfo{author}{T.~Pitassi}, \bibinfo{author}{C.~Dwork},
\newblock \bibinfo{title}{Learning fair representations},
\newblock in: \bibinfo{editor}{S.~Dasgupta}, \bibinfo{editor}{D.~McAllester} (Eds.), \bibinfo{booktitle}{Proceedings of the 30th International Conference on Machine Learning}, volume~\bibinfo{volume}{28} of \textit{\bibinfo{series}{Proceedings of Machine Learning Research}}, \bibinfo{publisher}{PMLR}, \bibinfo{address}{Atlanta, Georgia, USA}, \bibinfo{year}{2013}, pp. \bibinfo{pages}{325--333}. \URLprefix \url{https://proceedings.mlr.press/v28/zemel13.html}.
\bibitem[{Kamishima et~al.(2012)Kamishima, Akaho, Asoh, and Sakuma}]{kamishima}
\bibinfo{author}{T.~Kamishima}, \bibinfo{author}{S.~Akaho}, \bibinfo{author}{H.~Asoh}, \bibinfo{author}{J.~Sakuma},
\newblock \bibinfo{title}{Fairness-aware classifier with prejudice remover regularizer},
\newblock \bibinfo{year}{2012}, pp. \bibinfo{pages}{35--50}. \DOIprefix\doi{10.1007/978-3-642-33486-3_3}.
\bibitem[{Hooker(2021)}]{Movingbeyond}
\bibinfo{author}{S.~Hooker},
\newblock \bibinfo{title}{Moving beyond algorithmic bias is a data problem},
\newblock \bibinfo{journal}{Patterns} \bibinfo{volume}{2} (\bibinfo{year}{2021}).
\bibitem[{Sengamedu and Pham(2023)}]{sengamedu2023fairlabel}
\bibinfo{author}{S.~H. Sengamedu}, \bibinfo{author}{H.~Pham}, \bibinfo{title}{Fairlabel: Correcting bias in labels}, \bibinfo{year}{2023}. \href{http://arxiv.org/abs/2311.00638}{{\tt arXiv:2311.00638}}.
\bibitem[{Jiang and Nachum(2019)}]{jiang2019identifying}
\bibinfo{author}{H.~Jiang}, \bibinfo{author}{O.~Nachum}, \bibinfo{title}{Identifying and correcting label bias in machine learning}, \bibinfo{year}{2019}. \href{http://arxiv.org/abs/1901.04966}{{\tt arXiv:1901.04966}}.
\bibitem[{Northcutt et~al.(2021)Northcutt, Athalye, and Mueller}]{northcutt2021pervasive}
\bibinfo{author}{C.~G. Northcutt}, \bibinfo{author}{A.~Athalye}, \bibinfo{author}{J.~Mueller}, \bibinfo{title}{Pervasive label errors in test sets destabilize machine learning benchmarks}, \bibinfo{year}{2021}. \bibinfo{note}{Preprint at \url{https://arxiv.org/pdf/2103.14749.pdf}}.
\bibitem[{Olteanu et~al.(2019)Olteanu, Castillo, Diaz, and Kiciman}]{olteanu2019social}
\bibinfo{author}{A.~Olteanu}, \bibinfo{author}{C.~Castillo}, \bibinfo{author}{F.~Diaz}, \bibinfo{author}{E.~Kiciman},
\newblock \bibinfo{title}{Social data: Biases, methodological pitfalls, and ethical boundaries},
\newblock \bibinfo{journal}{Frontiers in Big Data} \bibinfo{volume}{2} (\bibinfo{year}{2019}).
\bibitem[{Fabbrizzi et~al.(2021)Fabbrizzi, Papadopoulos, Ntoutsi, and Kompatsiaris}]{fabbrizzi}
\bibinfo{author}{S.~Fabbrizzi}, \bibinfo{author}{S.~Papadopoulos}, \bibinfo{author}{E.~Ntoutsi}, \bibinfo{author}{I.~Kompatsiaris},
\newblock \bibinfo{title}{A survey on bias in visual datasets},
\newblock \bibinfo{journal}{CoRR} \bibinfo{volume}{abs/2107.07919} (\bibinfo{year}{2021}).
\bibitem[{Suresh and Guttag(2021)}]{suresh2021}
\bibinfo{author}{H.~Suresh}, \bibinfo{author}{J.~Guttag},
\newblock \bibinfo{title}{A framework for understanding sources of harm throughout the machine learning life cycle},
\newblock \bibinfo{journal}{Equity and Access in Algorithms, Mechanisms, and Optimization}  (\bibinfo{year}{2021}).
\bibitem[{CertNexus(2021)}]{nexus}
\bibinfo{author}{CertNexus}, \bibinfo{title}{Promote the ethical use of data-driven technologies}, \bibinfo{year}{2021}. \bibinfo{note}{\url{https://www.coursera.org/learn/promote-ethical-data-driven-technologies/lecture/5Ufbp/data-collection-bias}}.
\bibitem[{Centre~for Evidence-Based(2022)}]{catalogue}
\bibinfo{author}{U.~o.~O. Centre~for Evidence-Based}, \bibinfo{title}{Catalogue of bias}, \bibinfo{year}{2022}. \bibinfo{note}{\url{https://catalogofbias.org/biases/}}.
\bibitem[{D’Alessandro et~al.(2017)D’Alessandro, O’Neil, and LaGatta}]{d_Alessandro_2017}
\bibinfo{author}{B.~D’Alessandro}, \bibinfo{author}{C.~O’Neil}, \bibinfo{author}{T.~LaGatta},
\newblock \bibinfo{title}{Conscientious classification: A data scientist’s guide to discrimination-aware classification},
\newblock \bibinfo{journal}{Big Data} \bibinfo{volume}{5} (\bibinfo{year}{2017}) \bibinfo{pages}{120--134}.
\bibitem[{Saleiro et~al.(2018)Saleiro, Kuester, Hinkson, London, Stevens, Anisfeld, Rodolfa, and Ghani}]{saleiro2018aequitas}
\bibinfo{author}{P.~Saleiro}, \bibinfo{author}{B.~Kuester}, \bibinfo{author}{L.~Hinkson}, \bibinfo{author}{J.~London}, \bibinfo{author}{A.~Stevens}, \bibinfo{author}{A.~Anisfeld}, \bibinfo{author}{K.~T. Rodolfa}, \bibinfo{author}{R.~Ghani},
\newblock \bibinfo{title}{Aequitas: A bias and fairness audit toolkit},
\newblock \bibinfo{journal}{arXiv preprint arXiv:1811.05577}  (\bibinfo{year}{2018}).
\bibitem[{Northcutt et~al.(2021)Northcutt, Jiang, and Chuang}]{northcutt2021confidentlearning}
\bibinfo{author}{C.~G. Northcutt}, \bibinfo{author}{L.~Jiang}, \bibinfo{author}{I.~L. Chuang},
\newblock \bibinfo{title}{Confident learning: Estimating uncertainty in dataset labels},
\newblock \bibinfo{journal}{Journal of Artificial Intelligence Research (JAIR)} \bibinfo{volume}{70} (\bibinfo{year}{2021}) \bibinfo{pages}{1373--1411}.
\bibitem[{Angluin and Laird(1987)}]{DBLP:journals/ml/AngluinL87}
\bibinfo{author}{D.~Angluin}, \bibinfo{author}{P.~D. Laird},
\newblock \bibinfo{title}{Learning from noisy examples},
\newblock \bibinfo{journal}{Mach. Learn.} \bibinfo{volume}{2} (\bibinfo{year}{1987}) \bibinfo{pages}{343--370}.
\bibitem[{Batini and Scannapieco(2006)}]{DBLP:books/sp/dcsa/Batini06}
\bibinfo{author}{C.~Batini}, \bibinfo{author}{M.~Scannapieco}, \bibinfo{title}{Data Quality: Concepts, Methodologies and Techniques}, \bibinfo{publisher}{Springer}, \bibinfo{year}{2006}.
\bibitem[{Wang and Strong(1996)}]{DBLP:journals/jmis/WangS96}
\bibinfo{author}{R.~Y. Wang}, \bibinfo{author}{D.~M. Strong},
\newblock \bibinfo{title}{Beyond accuracy: What data quality means to data consumers},
\newblock \bibinfo{journal}{J. Manag. Inf. Syst.} \bibinfo{volume}{12} (\bibinfo{year}{1996}) \bibinfo{pages}{5--33}.
\bibitem[{Canali(2020)}]{stefano}
\bibinfo{author}{S.~Canali},
\newblock \bibinfo{title}{Towards a contextual approach to data quality},
\newblock \bibinfo{journal}{Data} \bibinfo{volume}{5} (\bibinfo{year}{2020}) \bibinfo{pages}{90}. \DOIprefix\doi{10.3390/data5040090}.
\bibitem[{Batini et~al.(2009)Batini, Cappiello, Francalanci, and Maurino}]{batini2009methodologies}
\bibinfo{author}{C.~Batini}, \bibinfo{author}{C.~Cappiello}, \bibinfo{author}{C.~Francalanci}, \bibinfo{author}{A.~Maurino},
\newblock \bibinfo{title}{Methodologies for data quality assessment and improvement},
\newblock \bibinfo{journal}{ACM computing surveys (CSUR)} \bibinfo{volume}{41} (\bibinfo{year}{2009}) \bibinfo{pages}{1--52}.
\bibitem[{Scannapieco and Catarci(2002)}]{scannap2002}
\bibinfo{author}{M.~Scannapieco}, \bibinfo{author}{T.~Catarci},
\newblock \bibinfo{title}{Data quality under a computer science perspective},
\newblock \bibinfo{journal}{Journal of The ACM - JACM} \bibinfo{volume}{2} (\bibinfo{year}{2002}).
\bibitem[{Pipino et~al.(2002)Pipino, Lee, and Wang}]{DBLP:journals/cacm/PipinoLW02}
\bibinfo{author}{L.~Pipino}, \bibinfo{author}{Y.~W. Lee}, \bibinfo{author}{R.~Y. Wang},
\newblock \bibinfo{title}{Data quality assessment},
\newblock \bibinfo{journal}{Commun. {ACM}} \bibinfo{volume}{45} (\bibinfo{year}{2002}) \bibinfo{pages}{211--218}.
\bibitem[{Rula et~al.(2014)}]{rula2014time}
\bibinfo{author}{A.~Rula}, et~al.,
\newblock \bibinfo{title}{Time-related quality dimensions in linked data}  (\bibinfo{year}{2014}).
\bibitem[{Zaveri et~al.(2016)Zaveri, Rula, Maurino, Pietrobon, Lehmann, and Auer}]{rula2016quality}
\bibinfo{author}{A.~Zaveri}, \bibinfo{author}{A.~Rula}, \bibinfo{author}{A.~Maurino}, \bibinfo{author}{R.~Pietrobon}, \bibinfo{author}{J.~Lehmann}, \bibinfo{author}{S.~Auer},
\newblock \bibinfo{title}{Quality assessment for linked data: A survey},
\newblock \bibinfo{journal}{Semantic Web} \bibinfo{volume}{7} (\bibinfo{year}{2016}) \bibinfo{pages}{63--93}.
\bibitem[{Spiel et~al.(2019)Spiel, Keyes, and Barlas}]{patching}
\bibinfo{author}{K.~Spiel}, \bibinfo{author}{O.~Keyes}, \bibinfo{author}{P.~Barlas},
\newblock \bibinfo{title}{Patching gender: Non-binary utopias in hci},
\newblock \bibinfo{year}{2019}, pp. \bibinfo{pages}{1--11}. \DOIprefix\doi{10.1145/3290607.3310425}.
\bibitem[{Nielsen(2020)}]{practical}
\bibinfo{author}{A.~Nielsen}, \bibinfo{title}{Practical Fairness}, \bibinfo{publisher}{O'Reilly Media, Inc.}, \bibinfo{year}{2020}. \URLprefix \url{http://gen.lib.rus.ec/book/index.php?md5=F9752B2F9693C98855A51504FE224DF6}.
\bibitem[{Batini et~al.(2015)Batini, Rula, Scannapieco, and Viscusi}]{BaSc}
\bibinfo{author}{C.~Batini}, \bibinfo{author}{A.~Rula}, \bibinfo{author}{M.~Scannapieco}, \bibinfo{author}{G.~Viscusi},
\newblock \bibinfo{title}{From data quality to big data quality},
\newblock \bibinfo{journal}{Journal of Database Management} \bibinfo{volume}{26} (\bibinfo{year}{2015}) \bibinfo{pages}{60--82}. \DOIprefix\doi{10.4018/JDM.2015010103}.
\bibitem[{Black and Van~Nederpelt(2020)}]{DAMA}
\bibinfo{author}{A.~Black}, \bibinfo{author}{P.~Van~Nederpelt}, \bibinfo{title}{Dimensions of Data Quality (DDQ)}, \bibinfo{publisher}{DAMA NL Foundation}, \bibinfo{year}{2020}.
\bibitem[{Angwin et~al.(2016)Angwin, Larson, Mattu, and Kirchner}]{propublica}
\bibinfo{author}{J.~Angwin}, \bibinfo{author}{J.~Larson}, \bibinfo{author}{S.~Mattu}, \bibinfo{author}{L.~Kirchner},
\newblock \bibinfo{title}{Machine bias},
\newblock \bibinfo{journal}{ProPublica}  (\bibinfo{year}{2016}).
\bibitem[{Edenberg and Wood(2023)}]{Edenberg2023-EDEAEL}
\bibinfo{author}{E.~Edenberg}, \bibinfo{author}{A.~Wood},
\newblock \bibinfo{title}{An epistemic lens on algorithmic fairness},
\newblock in: \bibinfo{booktitle}{Eaamo '23: Proceedings of the 3Rd Acm Conference on Equity and Access in Algorithms, Mechanisms, and Optimization}, \bibinfo{year}{2023}, pp. \bibinfo{pages}{1--10}.
\bibitem[{Ruberg and Ruelos(2020)}]{ruberg}
\bibinfo{author}{B.~Ruberg}, \bibinfo{author}{S.~Ruelos},
\newblock \bibinfo{title}{Data for queer lives: How {LGBTQ} gender and sexuality identities challenge norms of demographics},
\newblock \bibinfo{journal}{Big Data and Society} \bibinfo{volume}{7} (\bibinfo{year}{2020}). \DOIprefix\doi{10.1177/2053951720933286}.

\end{thebibliography}


\end{document}